\documentclass{article}

\usepackage{PRIMEarxiv}

\usepackage{url}

\usepackage{booktabs}       
\usepackage{amsfonts}       
\usepackage{nicefrac}       
\usepackage{microtype}      
\usepackage{xcolor}         

\usepackage{amsmath}
\usepackage{amssymb}
\usepackage{amsthm}
\usepackage{subcaption}
\usepackage{tabularx}
\usepackage{array}

\usepackage{graphicx}

\newtheorem{proposition}{Proposition}
\theoremstyle{remark}
\newtheorem{remark}{Remark}

\title{Rethinking Vacuity for OOD Detection in Evidential Deep Learning
}

\author{
  Claire McNamara\thanks{ORCID ID: \texttt{0000-0002-8263-8694, \texttt{mcnamacl@tcd.ie}}} \\
  Accenture Labs \\
  Ireland \\
}

\begin{document}
\maketitle
\begin{abstract}
Vacuity, or Uncertainty Mass (UM), is commonly used as a metric to evaluate Out-of-Distribution (OOD) detection in Evidential Deep Learning (EDL). It generally involves dividing the number of classes ($K$) by the total strength of belief ($S$) of the model's predictions, where $S$ is derived from summing the Dirichlet parameters. As such, UM is sensitive to the cardinality of $K$. As a result, when comparing In Distribution (ID) and OOD results, it is important that $K_{\mathrm{ID}}$ and $K_{\mathrm{OOD}}$ are equal; something that is not always ensured in practice. We provide an empirical demonstration of how results for AUROC and AUPR can substantially differ when class cardinality between ID and OOD differs by 1, with AUROC differing by as much as 0.346 and AUPR by 0.634 for standard EDL, and AUROC by 0.427 and AUPR by 0.745 for IB-EDL (both from Implementation B, Llama3-8B, ARC-E). Our findings isolate an evaluation artefact: when $K$ differs between ID and OOD, AUROC/AUPR can be artificially inflated without any change in model predictions. We further discuss the evaluation of EDL over causal language models using Multiple-Choice Question-Answer (MCQA) datasets and argue for clearer definitions of ID and OOD in this context. Our primary contribution is an empirical and theoretical demonstration that vacuity-based OOD detection in EDL-fine-tuned LLMs is highly sensitive to uncontrolled differences in evaluated class cardinality.
\end{abstract}

\section{Introduction}
Evidential Deep Learning (EDL) \cite{sensoy_evidential_2018} has emerged as a promising way to quantify uncertainty in causal large language models (LLMs) by fine-tuning them to operate as next token classifiers over Multiple-Choice Question-Answer (MCQA) datasets \cite{li_calibrating_2025, chun_evidential_2026}. The utilisation of datasets like MCQAs rather than free-text generation stems from how EDL operates. It requires a fixed set of classes ($K$) over which to form Dirichlet parameters. Summing the Dirichlet parameters gives you the model's total strength of belief ($S$) and from this, vacuity can then be understood as the uncertainty mass (UM), $\frac{K}{S}$. In EDL, UM is the quantity, vacuity its interpretation as ignorance-based uncertainty, and epistemic uncertainty the broader concept it approximates; for this paper we use \textit{UM} and \textit{vacuity} interchangeably.

While calculating vacuity-based uncertainty using $\frac{K}{S}$ can be useful \cite{sensoy_evidential_2018, aguilar_continual_2023, aguilar_cedl_2025}, as we will present in this paper, (i) vacuity-based Out-of-Distribution (OOD) detection can be highly sensitive to class number ($K$) specification in MCQA-based EDL settings, and (ii) comparisons across datasets therefore require $K$ to be held constant. This issue is particularly salient in MCQA-based LLM evaluation because the class set is not a fixed semantic label space in the same way as standard image classification (where EDL originates \cite{sensoy_evidential_2018}). The labels, such as A - D, are placeholders whose semantic content changes with each question, and the number of answer options can differ across datasets. Therefore, class-dependent uncertainty quantities may reflect changes in task format or answer-option structure rather than only epistemic uncertainty about the input.

Although the dependence of $\frac{K}{S}$ on $K$ is straightforward, its consequences in MCQA-based EDL evaluation have gone largely unexamined. We reproduce a published account \cite{li_calibrating_2025} whose released code scores ID and OOD inputs under different $K$, which inflates reported OOD-detection performance with no change to the underlying model; the same class-cardinality setting also appears in related implementations \cite{yang_bayesian_2024, rahmati_c-lora_2025}. Our focus is therefore on surfacing, quantifying, and mitigating this effect: a metric inflated by an output-space choice undermines trust in reported OOD performance.

This paper makes four contributions:
\begin{enumerate}
    \item We empirically demonstrate that mismatched effective class cardinality can substantially inflate AUROC and AUPR for vacuity-based OOD detection in EDL-fine-tuned LLMs.
    \item We prove that matched-$K$ expansion preserves the ranking of scores and so leaves AUROC and AUPR exactly unchanged, whereas mismatched expansion (in its numerator-only form) rescales one group's scores by a constant. The inflation therefore reflects a deterministic property of the evaluation rather than a change in uncertainty quality.
    \item We discuss why MCQA-based LLM evaluation differs from fixed-label classification and why clearer definitions of ID, OOD, and task-format shift are needed when applying EDL to LLMs.
    \item We provide a set of practical recommendations for practitioners utilising vacuity-based OOD evaluation.
\end{enumerate}

\section{Background and Related Work}
It is common to fine-tune an LLM using Parameter Efficient Fine-Tuning (PEFT) methods like LoRA \cite{hu_lora_2021} on MCQA datasets when examining how they perform with EDL \cite{li_calibrating_2025, chun_evidential_2026}. Working this way yields a constrained set of possible next token options - the options (classes) of the MCQA dataset. This section details how EDL operates over MCQA datasets with LLMs and outlines a variation of EDL - Information-Bottleneck EDL (IB-EDL) \cite{li_calibrating_2025}.

\subsection{Evidential Deep Learning (Standard EDL)}
In standard EDL, the output head applies a non-negative activation (e.g.\ \textit{Softplus}) to the target logits (those corresponding to the MCQA answer options). This generates the non-negative evidence $e_i$ the model assigns each class, from which the Dirichlet parameters are $\alpha_i = e_i + 1$. The prediction is the class with the highest expected probability, $\hat{y} = \arg\max_i \mathbb{E}[p_i] = \arg\max_i \alpha_i/S$, where $S=\sum_{i=1}^{K}\alpha_i$ is the total strength of belief (we use $S$ and $S_K$ interchangeably for the total Dirichlet strength). As mentioned, vacuity (\textit{UM}) is then $\frac{K}{S}$. The model is fine-tuned with the standard EDL objective \cite{sensoy_evidential_2018}, $\mathcal{L}_{\mathrm{EDL}} = \mathcal{L}_{\mathrm{MSE}} + \lambda\,\mathcal{L}_{\mathrm{Reg}}(\theta)$, combining a Dirichlet-based MSE (mean squared error) term with a KL regulariser that penalises evidence assigned to incorrect classes.

\subsection{Information-Bottleneck EDL (IB-EDL)}
A known limitation of standard EDL is that, because the model is only penalised for evidence on incorrect classes, the evidence assigned to the correct class is unbounded, which can produce an overconfident model. IB-EDL \cite{li_calibrating_2025} addresses this with an information bottleneck that encourages the model to retain only the evidence necessary for a correct prediction. Its objective, $\min_\theta \mathcal{L}_{\mathrm{IB\text{-}MSE}} + \beta\,\mathcal{L}_{\mathrm{IB\text{-}Info}}(\theta)$, re-derives the EDL MSE term within the Information Bottleneck framework and adds a term $\mathcal{L}_{\mathrm{IB\text{-}Info}}$ that regularises the information retained in the latent representation, thereby constraining the evidence available for prediction.

\subsection{Metrics Used in Evaluation}
There are a variety of methods used to assess a fine-tuned LLM's performance when using EDL or one of its variations. One of the most commonly used ones is Expected Calibration Error (ECE) \cite{posocco2021estimating, li_calibrating_2025, nemani_efficient_2025}. Another method used to assess model performance in general across modalities (Computer Vision, CV, for instance) is an assessment of how well the model performs when conducting Out-of-Distribution (OOD) detection. This is typically carried out by using the Area Under the Receiver Operating Characteristic curve (AUROC) \cite{fawcett2006introduction} or Area Under the Precision-Recall curve (AUPR) \cite{davis2006relationship} using metrics like UM (\textit{vacuity}) with ID and OOD inputs \cite{li_calibrating_2025, aguilar_cedl_2025, chen_revisiting_2024}. We raise concerns over evaluating model performance using OOD with UM in this manner as will be presented below. Broader skepticism over vacuity-based uncertainty measured as $\frac{K}{S}$ has been raised by Gao et al. describing the formulation as "somewhat arbitrary" \cite{Gao_survey_edl} noting that it does not follow from subjective logic and may not achieve practical effectiveness.

\section{Impact of Altering $K$}
Because AUROC and AUPR depend on scores only through their ranking, the relevant question is not whether vacuity changes under class expansion but whether it changes \emph{differently} for ID and OOD. The two propositions below separate these cases.

\begin{proposition}[Matched expansion preserves AUROC and AUPR exactly]
\label{prop:matched}
Let $u_K = K/S_K$ with $S_K=\sum_{i=1}^{K}\alpha_i$. Suppose $m$ zero-evidence classes ($e_{K+j}{=}0$, i.e.\ $\alpha_{K+j}{=}1$) are appended to \emph{every} example, ID and OOD alike.
Then
\[
u_{K+m} = \frac{K+m}{S_K+m} = \frac{u_K\,(K+m)}{K + m\,u_K},
\]
which is strictly increasing in $u_K$ and depends only on $K$ and $m$ (Appendix~\ref{sec:proof}). The same order-preserving map is therefore applied to every example, and AUROC and AUPR are \emph{exactly} unchanged.
\end{proposition}

\begin{proposition}[Mismatched expansion rescales one group]
\label{prop:mismatched}
If the evaluated cardinality is inflated for the OOD set only, holding $S$ fixed, then $u^{\mathrm{OOD}} \mapsto (K_{\mathrm{OOD}}/K_{\mathrm{ID}})\,u^{\mathrm{OOD}}$ while ID scores are untouched. The evaluation multiplies one group's scores by a constant (e.g., $1.25$ for $K_{\mathrm{ID}}{=}4 \to K_{\mathrm{OOD}}{=}5$); the induced change in AUROC is determined entirely by the overlap of the two vacuity distributions.
\end{proposition}

\begin{remark}[Pointwise invariance]
Pointwise invariance is a strictly stronger and unnecessary condition: $u_{K+1}=u_K$ holds iff $\alpha_{K+1}=S_K/K$, equivalently $e_{K+1}=S_K/K-1$ (Appendix~\ref{sec:proof}). Proposition~\ref{prop:matched} shows the metrics are preserved even though this condition fails.
\end{remark}

Both propositions are deliberately narrow. Proposition~\ref{prop:matched} concerns appended zero-evidence classes; it does not imply that two datasets with genuinely different answer-option counts become comparable once the evaluated $K$ is equalised, since matching $K$ there requires altering the evaluation set itself (Section~\ref{sec:impA}). Proposition~\ref{prop:mismatched} describes the numerator-only construction used for Implementation B, and does not describe Implementation A: there the spurious fifth class carries the model's actual "E"-token evidence, so $S$ changes per example and the re-scoring is not a constant multiple. The two constructions are separated in Section~\ref{sec:impB}.

\subsection{Isolating K}
Table~\ref{tab:k_expansion_results} presents results for predictions generated with standard EDL using our code (Implementation B, detailed in the next section), with Llama3-8B fine-tuned on the OBQA \cite{OpenBookQA2018} dataset as ID and ARC-C as OOD \cite{clark2018think}, where $K$ is isolated. The model predictions are held fixed and we recompute the reciprocal of UM (measured as $\frac{1}{UM}$) under varying evaluated $K$; any change in AUROC or AUPR (ID=1, OOD=0) is therefore a deterministic property of the evaluation, not of model behaviour. \textit{OOD-only} expansion refers to where the set of evidence for the OOD example only has been expanded to include $x$ additional evidence terms where $K + x$ and $0 \leq x \leq 4$. The appended evidence is $0$ (as a reminder, $\alpha_i = e_i + 1$ and thus, expansion still changes both $K$ and $S$, but not in the proportional manner required for vacuity invariance).\footnote{This illustrative expansion appends explicit zero-evidence classes (changing both $K$ and $S$); the per-dataset results in Appendix~\ref{sec:detailed_results} instead hold $S$ fixed and vary only the numerator $K$. The two yield near-identical AUROC ($0.851$ vs $0.859$), confirming the effect is a property of the $K$ mismatch, not the construction} \textit{Matched} expansion refers to where the evidence has been appended to both ID and OOD examples. As is evident by the results, when $K$ increases for OOD-only, the results show a substantial (unjustified) improvement even when $K$ has only increased by 1 (AUROC: +0.270, AUPR: +0.507). This demonstrates that vacuity-based OOD performance can be dominated by evaluated class cardinality rather than uncertainty quality. 

\begin{table}[!ht]
\centering
\caption{Effect of class-cardinality expansion on OOD detection performance using OBQA $\rightarrow$ ARC-C. Deltas are computed relative to the baseline ($K_{\text{ID}}{=}4, K_{\text{OOD}}{=}4$). Values of $K$ used as the effective evaluated output dimensionality when computing evidence, $S$, and OOD scores.}
\label{tab:k_expansion_results}
\begin{tabular}{lcccccc}
\hline
\textbf{Condition} & $K_{\text{ID}}$ & $K_{\text{OOD}}$ & \textbf{AUROC} & $\Delta$ & \textbf{AUPR} & $\Delta$ \\
\hline
Baseline              & 4 & 4 & \textbf{0.581} & 0.000 & 0.346 & 0.000 \\
OOD-only expansion    & 4 & 5 & 0.851 & $\uparrow$ 0.270 & 0.853 & $\uparrow$ 0.507 \\
OOD-only expansion    & 4 & 6 & 0.898 & $\uparrow$ 0.316 & 0.912 & $\uparrow$ 0.566 \\
OOD-only expansion    & 4 & 7 & 0.914 & $\uparrow$ 0.333 & 0.929 & $\uparrow$ 0.583 \\
OOD-only expansion    & 4 & 8 & 0.926 & $\uparrow$ 0.345 & 0.939 & $\uparrow$ 0.593 \\
Matched expansion     & 5 & 5 & 0.581 & 0.000 & 0.346 & 0.000 \\
Matched expansion     & 6 & 6 & 0.581 & 0.000 & 0.346 & 0.000 \\
Matched expansion     & 7 & 7 & 0.581 & 0.000 & 0.346 & 0.000 \\
Matched expansion     & 8 & 8 & 0.581 & 0.000 & 0.346 & 0.000 \\
\hline
\end{tabular}
\end{table}

\section{Impact in Practice}
\subsection{OOD using UM and MP}
A motivating example of the unreliability of using vacuity-based uncertainty as a means of detecting OOD inputs in practice can be seen through a faithful reproduction of the work of the authors of IB-EDL \cite{li_calibrating_2025}. As discussed above, IB-EDL aims to reduce spurious evidence by applying a \textit{bottleneck} regularisation over the evidence. The results show that when Llama3-8B \cite{grattafiori2024llama3herdmodels} is fine-tuned on variations of EDL and equivalent methods (e.g., R-EDL \cite{chen_revisiting_2024} and I-EDL \cite{deng_uncertainty_2023}), IB-EDL generally maintains or increases accuracy across the MCQA datasets tested and reports lower ECE. This suggests that IB-EDL does indeed work better in practice when calibrating an LLM.

An important consideration that arose during reproduction however was that OOD detection was achieved using the reciprocal of UM ($\frac{1}{UM}$) and Max Probability (MP) \cite{hendrycks2017a} with ID=1 and OOD=0. AUROC results from \cite{li_calibrating_2025} report that both using the UM and MP results in strong OOD detection across EDL methods, models, and datasets indicating that both metrics are useful when used in the context of detecting OOD inputs. Specifically, the published AUROC for reciprocal of UM across all tested MCQA datasets for Llama3-8B with IB-EDL was reported to be between 85.45 and 94.77, and the MP was reported to be between 88.14 and 89.16. During reproduction, while it was possible to achieve close to the reported number using the authors' provided GitHub implementation \footnote{https://github.com/sandylaker/ib-edl/}, it was not possible to recreate close to the reported numbers for UM nor MP using an independent implementation. A detailed account of the full reproduction procedure can be found in the Appendix \ref{sec:implementation_procedure}. This prompted an inspection of the released dataset helper code that indicated that for all OOD tests in the main paper, while the number of classes ($K$) for the ID dataset was four, the number of classes for the OOD datasets tested was five. The result was a much larger difference between ID and OOD uncertainties and probabilities. This highlighted two issues. First, vacuity-based uncertainty is sensitive to $K$ (naturally, as $\frac{K}{S}$): when class cardinality differs, the result is a much larger separation between ID and OOD in AUROC. Second, once corrected, UM and MP detect OOD only weakly for the datasets examined in the main paper. We include MP for completeness, in line with IB-EDL \cite{li_calibrating_2025}, but it is not governed by the same $\frac{K}{S}$ dependence as vacuity and is not a primary object of analysis.

\subsection{Implementation Details}
In order to perform a faithful reproduction of the work of \cite{li_calibrating_2025}, we used both the code associated with the paper (which we will refer to as Implementation A) and our own separate implementation (which we will refer to as Implementation B) using both IB-EDL and standard EDL to fine-tune Llama3-8B \cite{grattafiori2024llama3herdmodels} on the OBQA dataset \cite{OpenBookQA2018} (following \cite{li_calibrating_2025}) with the same training arguments and number of training steps (exact details in the Appendix \ref{sec:implementation_procedure}). We further used the same OOD datasets as the main paper \cite{li_calibrating_2025} to test the fine-tuned models: ARC-E \cite{clark2018think}, ARC-C \cite{clark2018think}, CSQA \cite{talmor-etal-2019-commonsenseqa}. We separately test MMLU Math \cite{hendryckstest2021} as was tested by Li et al. \cite{li_calibrating_2025} in their Appendix as an example of a dataset from a different domain (maths vs. general knowledge QA). Table~\ref{tab:k_id_ood} shows two distinct sources of ID/OOD mismatch: ARC-C and ARC-E are four-option tasks but Implementation~A scores them at $K{=}5$ (we use $K{=}4$ in Implementation~B), whereas CSQA is scored at its correct $K{=}5$ yet, as a five-option task, is mismatched against the four-option ID. Note that MMLU Math is itself a four-option task, so its correct cardinality ($K{=}4$) already matches OBQA; we additionally report a $K{=}5$ mismatch probe for it (as for ARC) to show that the same class-cardinality effect arises even for a far-OOD dataset. Further details about the ARC datasets label choice frequency can be found in the Appendix~\ref{sec:arc_answer_key_cardinality}.

\begin{table}[ht]
\centering
\caption{Answer-option cardinality per dataset. OOD detection is comparable only when $K_{\mathrm{OOD}}{=}K_{\mathrm{ID}}{=}4$ (OBQA); \textbf{bold} marks $K\neq4$.}
\label{tab:k_id_ood}
\begin{tabular}{llccc}
\hline
Dataset & Role & Native $K$ & Impl.\ A & Impl.\ B \\
\hline
OBQA      & ID  & 4 & 4 & 4 \\
ARC-C     & OOD & 4 & \textbf{5} & 4 \\
ARC-E     & OOD & 4 & \textbf{5} & 4 \\
CSQA      & OOD & \textbf{5} & \textbf{5} & \textbf{5} \\
MMLU Math & OOD & 4 & 4 & 4 \\
\hline
\end{tabular}
\end{table}

\subsection{Results of Implementation A}
\label{sec:impA}
As previously stated, Implementation A refers to the code provided by the authors of IB-EDL~\cite{li_calibrating_2025}. In reviewing Implementation A, we observed that the variable for K (n\_labels) had been set to 5 rather than 4 for ARC-C and ARC-E. We note that prior implementations \cite{yang_bayesian_2024, li_calibrating_2025, rahmati_c-lora_2025} use \texttt{n\_labels}\,$=$\,5 for ARC-C and ARC-E, although $7{,}750$ of $7{,}787$ questions ($99.52\%$) have only four answer options and the dataset card specifies the label set as A - D \cite{clark2018think} (Appendix~\ref{sec:arc_answer_key_cardinality}). The setting appears to originate in the Laplace-LoRA codebase \cite{yang_bayesian_2024} (see the ARC config\footnote{\url{https://github.com/MaximeRobeyns/bayesian_lora/blob/4e99db8df6860eaf0e4bc1a168fb3faa45a094f0/examples/configs/dset/arc.yaml#L9}}) and has propagated to subsequent implementations, including IB-EDL \cite{li_calibrating_2025} and, more recently, C-LoRA \cite{rahmati_c-lora_2025} (NeurIPS 2025; see the C-LoRA ARC implementation\footnote{\url{https://github.com/ahra99/c_lora/blob/e706627cc5f62d25515b49d0a55c733ba31aec82/dataset/S2SDataset_Classification.py#L22}}). The mismatch is therefore a propagating implementation setting rather than a one-off occurrence, reinforcing the need for the per-dataset $K$ checks we recommend. Because the setting is applied uniformly across all methods compared within each of these works, it does not differentially affect the relative comparisons reported there; the concern raised here is instead with the interpretation of absolute vacuity-based OOD scores and with comparisons drawn across datasets or against previously reported values.

Re-running Implementation~A with $K{=}4$ for ARC-C and ARC-E drops both AUROC and AUPR substantially. For CSQA we instead match $K$ by excluding answers whose correct label is option-"E" so the four-option evaluation retains a valid correct label; as CSQA is used as OOD input rather than for accuracy, this isolates the effect of the evaluated output space. Standard EDL shows the same pattern as IB-EDL. Table~\ref{tab:um_summary} summarises the UM (vacuity) AUROC; the full MP, UM, and AUPR breakdown for both implementations is in Appendix~\ref{sec:detailed_results}.

These results are obtained using the authors' implementation, in which ARC-C and ARC-E were evaluated with five-dimensional evidence/concentration vectors despite having four answer options. This matters because the expected probability for each class is calculated as $\alpha_i/S$, where $S$ is the total Dirichlet strength over the evaluated class set. Increasing the effective output dimensionality can therefore lower the maximum expected probability for OOD examples, increasing separation from the four-class ID scores and thereby inflating AUROC and AUPR (as seen in both tables).

\begin{table}[!ht]
\centering
\small
\caption{Summary: UM (vacuity) AUROC under the mismatched-$K$ evaluation versus the corrected matched-$K$ evaluation, for both EDL variants and both implementations (ID: OBQA). Mismatched is $K_{\mathrm{OOD}}{=}5$ (CSQA: 5-class as-is); corrected is matched $K{=}4$ (CSQA: option "E" removed); $\Delta$ (corrected$-$mismatched, $\downarrow$) is the resulting drop. Full MP/UM/AUPR detail is in Appendix~\ref{sec:detailed_results}.}
\label{tab:um_summary}
\setlength{\tabcolsep}{4pt}
\begin{tabular}{l ccc ccc}
\toprule
& \multicolumn{3}{c}{Implementation A} & \multicolumn{3}{c}{Implementation B} \\
\cmidrule(lr){2-4}\cmidrule(lr){5-7}
Dataset & mismatch ($K$=5) & correct ($K$=4) & $\Delta$ & mismatch ($K$=5) & correct ($K$=4) & $\Delta$ \\
\midrule
\multicolumn{7}{l}{\textit{Standard EDL}} \\
MMLU Math & 0.951 & 0.918 & $\downarrow$0.032 & 0.952 & 0.898 & $\downarrow$0.054 \\
ARC-C     & 0.866 & 0.601 & $\downarrow$0.265 & 0.859 & 0.581 & $\downarrow$0.278 \\
ARC-E     & 0.855 & 0.546 & $\downarrow$0.309 & 0.840 & 0.494 & $\downarrow$0.346 \\
CSQA      & 0.873 & 0.650 & $\downarrow$0.222 & 0.845 & 0.593 & $\downarrow$0.252 \\
\midrule
\multicolumn{7}{l}{\textit{IB-EDL}} \\
MMLU Math & 0.979 & 0.865 & $\downarrow$0.114 & 0.939 & 0.723 & $\downarrow$0.216 \\
ARC-C     & 0.967 & 0.670 & $\downarrow$0.297 & 0.976 & 0.631 & $\downarrow$0.345 \\
ARC-E     & 0.961 & 0.596 & $\downarrow$0.365 & 0.972 & 0.545 & $\downarrow$0.427 \\
CSQA      & 0.941 & 0.597 & $\downarrow$0.344 & 0.962 & 0.687 & $\downarrow$0.275 \\
\bottomrule
\end{tabular}
\end{table}

One caveat applies specifically to the CSQA rows. Because CSQA is natively a five-option task, its four-class evaluation is produced by removing option-"E" examples, which changes both the evaluated cardinality and the composition of the OOD set (reflected in the shift of the AUPR base rate from $0.291$ to $0.337$). The CSQA $\Delta$ therefore conflates the class-cardinality effect with a change in the evaluation set. The ARC-C and ARC-E examples do not have this confound as the same examples are re-scored under $K{=}4$ and $K{=}5$ with the AUPR base rate held fixed within each dataset ($0.299$ for ARC-C, $0.174$ for ARC-E for Implementation A).
\subsection{Results of Implementation B}
\label{sec:impB}
Implementation B is our independent reimplementation of IB-EDL and standard EDL (setup in Appendix~\ref{sec:implementation_procedure}).

A core difference between implementations A and B is that predictions for Implementation B with ARC-C and ARC-E were generated using four classes rather than the five used in Implementation A. The mismatched $K{=}5$ setting for Implementation B is the numerator-only construction of Proposition~\ref{prop:mismatched}: only the $\frac{K}{S}$ ratio underlying UM changes. MP and normalised entropy ($H_{norm}$), computed from the unchanged four-class probabilities, are therefore exactly unaffected ($\Delta{=}0.000$ in Tables~\ref{tab:own_ibedl_num_classes_delta}, \ref{tab:own_std_edl_num_classes_delta}, and~\ref{tab:ne_std_ibedl_imp_b}). This contrasts with Implementation A, where the spurious fifth class carries the model's real "E"-token evidence and therefore also lowers MP and $H_{norm}$ (Tables~\ref{tab:ib_edl_num_classes_delta}, \ref{tab:standard_edl_num_classes_delta}, and \ref{tab:ne_std_ibedl_delta_implA}). The standard EDL results for Implementation B exhibit the same pattern. The weak matched-$K$ results for ARC-C and ARC-E should not necessarily be interpreted as evidence that the metric cannot detect any distribution shift. Rather, these datasets may constitute near-OOD \cite{winkens2020contrastive} shifts relative to OBQA because they remain MCQA science/reasoning datasets. In order to assess this, in line with \cite{li_calibrating_2025}, we also evaluate using the MMLU Math MCQA dataset \cite{hendryckstest2021}. The stronger performance on MMLU Math suggests that UM and MP may be more informative under larger domain shifts. This reinforces the need to distinguish dataset shift, domain shift, and task-format shift when evaluating OOD detection in MCQA-based EDL, something which is discussed in more detail below. To confirm that these findings are not specific to Llama3-8B, Appendix~\ref{sec:qwen_impB} repeats the entire Implementation B pipeline with a Qwen3.5-9B-Base \cite{qwen3.5-base} backbone and observes the same class-cardinality pattern.

\subsection{Restoring Comparability}
The correction is straightforward: score ID and OOD under identical $K$. This in practice means verifying that \texttt{n\_labels} matches the number of answer options for every dataset. Re-scoring the same frozen predictions under the correct $K$ collapses the inflated separation back to its honest value. Figure~\ref{fig:um_dist} makes this concrete for Standard EDL on OBQA$\rightarrow$ARC-E under Implementation~A: under the mismatch ($K_{\mathrm{OOD}}{=}5$) the OOD vacuity distribution is visibly shifted away from ID (AUROC $=0.855$), whereas under matched $K{=}4$ the distributions overlap and AUROC falls to $0.546$ - close to the near-OOD chance level expected for two MCQA science-reasoning datasets.

We stress, however, that a superficially "$K$-normalised" score is \emph{necessary but not sufficient}. Under Implementation~B's numerator-only mismatch, $K$ enters the vacuity numerator alone, so normalised entropy ($H/\log_2 K$) is invariant; under the realistic Implementation~A mismatch it is inflated by up to $0.311$ AUROC (Appendix~\ref{sec:ent_analysis}), because the added class also perturbs the underlying evidence. Matching $K$ is the reliable correction.

\begin{remark}
For the numerator-only mismatch (Implementation~B), Proposition~\ref{prop:mismatched} predicts that a constant rescale of one group's scores moves AUROC most where the distributions overlap most. Table~\ref{tab:um_summary} shows the consequence: under mismatched $K$ the four OOD datasets compress into a narrow band (IB-EDL, Implementation B: $0.939$ - $0.976$), whereas under matched $K$ they separate across $0.545$ - $0.723$. Near-OOD ARC-E and far-OOD MMLU Math are indistinguishable here ($0.972$ vs $0.939$) and separable once corrected ($0.545$ vs $0.723$). Mismatched-$K$ scoring therefore removes the metric's ability to discriminate between OOD scenarios. Consistently, the inflation is largest for the datasets whose matched-$K$ detector is weakest (ARC-E in every column of Table~\ref{tab:um_summary}).
\end{remark}

\begin{figure}[h!]
\centering
\includegraphics[width=0.92\linewidth]{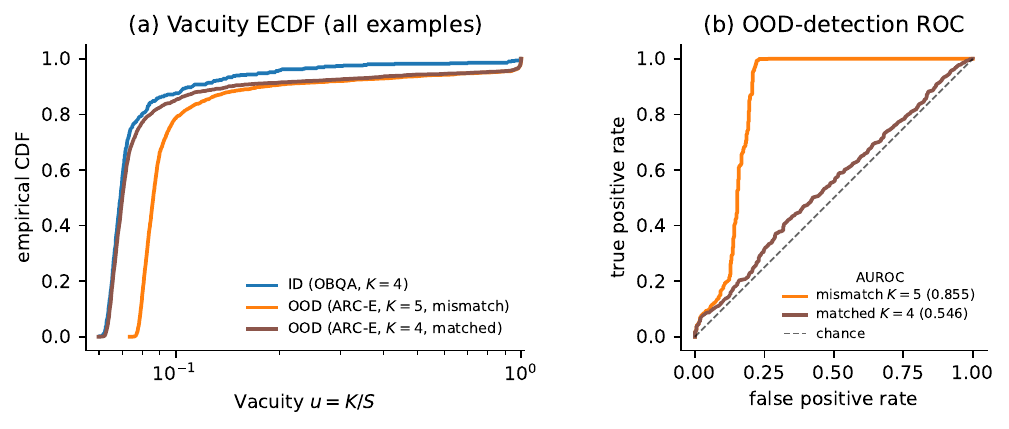}
\caption{Vacuity ($u=K/S$) for ID (OBQA) versus OOD (ARC-E) under Standard EDL (Implementation~A), using \emph{all} examples with no values clipped. \textbf{(a)} Empirical Cumulative Distribution Function (ECDF, log $x$-axis): under the mismatched $K_{\mathrm{OOD}}{=}5$ the OOD curve is shifted toward higher vacuity, whereas under the correct $K{=}4$ it nearly coincides with ID, including in the high-uncertainty tail. \textbf{(b)} The corresponding OOD-detection ROC curves: the mismatch inflates AUROC to $0.855$, while matched $K{=}4$ collapses it to $0.546$, close to the diagonal expected for two near-OOD MCQA science-reasoning datasets. Predictions are identical throughout; only the evaluated $K$ changes.}
\label{fig:um_dist}
\end{figure}

\begin{center}
\fbox{\parbox{0.95\linewidth}{%
\textbf{Practical recommendations for vacuity-based OOD evaluation.}
\begin{itemize}\setlength\itemsep{1pt}
  \item \textbf{Report $K_{\mathrm{ID}}$ and $K_{\mathrm{OOD}}$ explicitly} for every ID/OOD pair.
  \item \textbf{Verify \texttt{n\_labels}} matches the answer-option count of each dataset before scoring.
  \item \textbf{Prefer matched-$K$ evaluation}; treat any comparison with $K_{\mathrm{ID}}\neq K_{\mathrm{OOD}}$ as uninterpretable.
  \item \textbf{Do not rely on normalisation alone}; $K$-normalised scores remain sensitive to genuine cardinality mismatch.
\end{itemize}}}
\end{center}

\section{Discussion}
We do not argue that EDL or IB-EDL is ineffective as a calibration method and these findings as such should be interpreted as a clarification of evaluation conditions rather than a critique of prior work. When class cardinality is controlled, IB-EDL continues to exhibit strong calibration properties; however, OOD separability measured via class-dependent metrics requires careful interpretation. Indeed, our reproduced fine-tuning results show competitive accuracy, ECE, and NLL. Our claim is narrower: vacuity-based OOD comparisons can be misleading when the effective number of evaluated classes differs between ID and OOD datasets. This phenomenon is not unique to IB-EDL: as noted in Section~\ref{sec:impA}, the same $\texttt{n\_labels}{=}5$ setting is inherited from Laplace-LoRA \cite{yang_bayesian_2024} and propagates to later methods such as C-LoRA \cite{rahmati_c-lora_2025}. Class-cardinality sensitivity also extends beyond vacuity. For example, \cite{nemani_efficient_2025} evaluate a model trained on a 2-class dataset (Amazon Reviews \cite{amazon_reviews_dataset}) on a 10-class dataset (Yahoo Answers \cite{yahoo_answers_dataset}). While their work does not rely on vacuity, their uncertainty measures (predictive entropy, expected entropy, and mutual information) are defined over the output simplex and therefore scale with the number of classes $K$. As such, differences in uncertainty between ID and OOD settings may also reflect differences in output dimensionality rather than solely differences in data distribution. The underlying point is that evidential uncertainty measures are defined with respect to a particular output space. When the dimensionality or semantics of this space change between ID and OOD settings, the resulting uncertainty values are not strictly comparable, even when computed correctly.
The consideration of what metric to use when detecting OOD inputs raises a larger issue relating to the ambiguity of what constitutes OOD input in the context of an EDL evaluation with LLMs. In the context of an LLM fine-tuned on common-sense MCQA data, should we consider a different MCQA common-sense dataset truly OOD? While it may not reflect the precise data the model was trained on, it likely can still be considered within the realm of its ID distribution. In other words, it may be more apt to consider such datasets as \textit{near OOD} \cite{winkens2020contrastive} datasets. As seen above, when confronted with input data from a different domain to its training data (Math MMLU), MP and UM perform better as OOD detectors. Such datasets perhaps can be considered as \textit{far OOD} \cite{winkens2020contrastive} ones. Furthermore, due to the nature of LLMs and the method by which they are fine-tuned, there is another distinction we can make - \textit{task OOD}. If considering an LLM that has been fine-tuned to distinguish between four options of an MCQA input, can we consider providing it with five options instead as an alternate source of OOD? 

This is less of a concern in the CV setting EDL originates from, where the output space is fixed by the ID training labels (a CIFAR-10 classifier still emits CIFAR-10 outputs on CIFAR-100 inputs), so the explicit MCQA class-cardinality mechanism does not transfer directly. Shen et al. \cite{shen_are_2024} further argue that EDL may not capture epistemic (model) uncertainty in the Bayesian sense at all, and that where it succeeds downstream it may do so in spite of this (a caution directly relevant to interpreting vacuity as epistemic).

Underlying these considerations is an assumption that causal LLMs over MCQA are comparable to CV classifiers over images. They are not: in MCQA the output space is defined per input and class semantics are instance-dependent, so measures that depend on output dimensionality are not invariant to the number of answer options even when the task is unchanged. This makes MCQA an imperfect proxy for class-based uncertainty without explicit correction, echoing broader concerns that MCQA outputs may reflect decoding variability more than uncertainty \cite{zeng_uncertainty_2025} and that models may default to the "least incorrect option" \cite{wang2025llms}.

\section{Limitations}
There are a number of limitations associated with this paper. Firstly, the datasets used are MCQA only. While this is necessary to evaluate EDL in the context of LLMs, it is a limitation of the approach. Secondly, the evaluation is limited to two implementations of EDL - standard EDL and IB-EDL. As we provide an ablation keeping the predictions static however, we show how quickly the AUROC can change with mismatched cardinality which is something that is EDL-implementation-agnostic. Thirdly, and perhaps most importantly, is the limitation surfaced in Discussion - that using EDL with MCQAs may not be operating as expected in that EDL assumes that UM is epistemic however, that may not be the case. Finally, matched-$K$ levels come from one run per method/backbone: the $\Delta$ values are deterministic, but absolute levels are not -- IB-EDL differs by up to $0.142$ AUROC between our implementations (MMLU Math: $0.865$ vs $0.723$).

\section{Conclusion}
Our findings do not argue against evidential uncertainty itself, but against comparing vacuity-based OOD scores across mismatched class cardinalities. We argue that when considering using vacuity as UM for OOD detection, it is imperative that $K_{\mathrm{ID}}$ and $K_{\mathrm{OOD}}$ are equal and demonstrate how $K$-variance can impact results in practice. We further raise questions surrounding the direct transferability of using EDL in the context of causal language models over MCQAs and a need for greater clarity in what exactly is being thought of as OOD input. It is hoped that our work will contribute to greater transparency in how ID and OOD datasets are defined within the context of EDL with LLMs and result in greater care being taken when calculating and comparing UM as a metric for OOD detection. Our follow-up work aims to be able to clarify types of OOD input in this context and how they affect EDL in practice. 

\section*{Generative AI Disclosure}
OpenAI's ChatGPT-5 was used for sentence level polishing. Claude Code Opus 4.8 was used to scaffold the code used as part of the implementation and analysis. The authors take full responsibility for the manuscript and code and have verified all output thoroughly.

\section*{Acknowledgements}
The work undertaken in this paper was carried out during employment as an Applied AI Research Intern at Accenture Labs, Ireland.

\bibliographystyle{unsrt}  
\bibliography{bibliography, references}  


\appendix

\section{Proofs}
\label{sec:proof}

\subsection*{Proof of Proposition~\ref{prop:matched}}
\begin{proof}
AUROC and AUPR are rank statistics: for a fixed labelling of the pooled ID\,$\cup$\,OOD set, both depend on the scores only through their induced ordering. It therefore suffices to show that the re-scoring map is strictly increasing and applied identically to every example.

Appending $m$ classes with $e_{K+j}{=}0$ (so $\alpha_{K+j}{=}1$) to every example gives $S_{K+m}=S_K+m$, and substituting $S_K=K/u_K$,
\[
u_{K+m}=\frac{K+m}{S_K+m}=\frac{u_K\,(K+m)}{K+m\,u_K},
\qquad
\frac{\mathrm{d}u_{K+m}}{\mathrm{d}u_K}=\frac{K(K+m)}{(K+m\,u_K)^2}>0
\]
since $K,m\geq 1$ and $u_K>0$. The map depends only on $K$ and $m$, which are common to ID and OOD under matched expansion, so it is the same strictly increasing function for every example. The ordering of the pooled scores is preserved and AUROC and AUPR are exactly unchanged.
\end{proof}

\subsection*{Proof of Proposition~\ref{prop:mismatched}} 
\begin{proof}
Under the numerator-only construction the total strength $S$ and the underlying $K_{\mathrm{ID}}$-class probability vector are held fixed, and only the cardinality entering the numerator is inflated. The score assigned to an OOD example is therefore
\[
u'^{\,\mathrm{OOD}}=\frac{K_{\mathrm{OOD}}}{S}
=\frac{K_{\mathrm{OOD}}}{K_{\mathrm{ID}}}\cdot\frac{K_{\mathrm{ID}}}{S}
=\frac{K_{\mathrm{OOD}}}{K_{\mathrm{ID}}}\;u^{\,\mathrm{OOD}},
\]
while ID examples are scored with $K_{\mathrm{ID}}$ throughout and are unchanged. The evaluation thus multiplies one group's scores by the constant $K_{\mathrm{OOD}}/K_{\mathrm{ID}}$ -- $1.25$ for $K_{\mathrm{ID}}{=}4 \to K_{\mathrm{OOD}}{=}5$. Because MP and normalised entropy are computed from the unchanged $K_{\mathrm{ID}}$-class probability vector, they are exactly invariant under this construction, which is what Tables~\ref{tab:own_ibedl_num_classes_delta}, \ref{tab:own_std_edl_num_classes_delta}, and~\ref{tab:ne_std_ibedl_imp_b} report as $\Delta{=}0.000$.
\end{proof}

\subsection*{Proof of Remark 1}
\begin{proof}
Setting $u_{K+1}=u_K$ with $u_{K+1}=(K+1)/(S_K+\alpha_{K+1})$ and $u_K=K/S_K$ gives $(K+1)S_K = K(S_K+\alpha_{K+1})$, i.e.\ $S_K = K\alpha_{K+1}$, so $\alpha_{K+1}=S_K/K$. Since $\alpha_{K+1}=e_{K+1}+1$, this is $e_{K+1}=S_K/K-1$: vacuity is invariant under the added class iff it receives the mean concentration of the original $K$ classes.
\end{proof}
\section{ARC Answer-Option Cardinality}
\label{sec:arc_answer_key_cardinality}
To substantiate that $K{=}4$ is the correct evaluated cardinality for ARC, we counted the number of answer options per question across all six ARC-Challenge and ARC-Easy splits (Table~\ref{tab:arc_answerkey_freq}). The dataset card declares the label set as A - D \cite{clark2018think}, and $7{,}750$ of the $7{,}787$ questions ($99.52\%$) indeed have exactly four options. Of the remainder, $22$ questions ($0.28\%$) have three options and $15$ ($0.19\%$) have five. Note that a minority of questions use the numeric encoding 1 - 4 in place of A - D; these are the same four-option format under a different label set and are counted as four-option questions here. Setting \texttt{n\_labels}$=5$ for ARC therefore introduces a spurious fifth class for $99.81\%$ of examples rather than accommodating a genuine five-option task, confirming that $K{=}4$ is the appropriate evaluated output dimensionality.

\begin{table}[!ht]
\centering
\caption{Number of answer options per question, pooled across the
ARC-Challenge and ARC-Easy train, validation, and test splits ($n = 7{,}787$).}
\label{tab:arc_answerkey_freq}
\begin{tabular}{rrr}
\toprule
\textbf{Options per question} & \textbf{Questions} & \textbf{\%} \\
\midrule
3 & \phantom{7,}22 & \phantom{9}0.28 \\
4 & 7{,}750 & 99.52 \\
5 & \phantom{7,}15 & \phantom{9}0.19 \\
\bottomrule
\end{tabular}
\end{table}

\section{Implementation Procedure}
\label{sec:implementation_procedure}
The same hyperparameters were used for both Implementation A and Implementation B. Both the fine-tuning of Llama3-8B and the generation of predictions for ID and OOD inputs were carried out on a Mac Studio (Apple M3 Ultra, 512\,GB unified memory) using the Metal Performance Shaders (MPS) backend. We use float32 rather than float16 on MPS because the EDL losses (which involve softplus, logarithm, digamma, and lgamma terms) overflow the limited range of float16 and produce NaN gradients, causing training to collapse; the Mac Studio's unified memory is more than sufficient for float32. To aid reproducibility, all random number generators (Python, NumPy, and PyTorch across CPU/CUDA/MPS) are seeded with a fixed value of 42, and the same seed is applied at inference so that the stochastic IB-EDL noise draws are deterministic. Following \cite{li_calibrating_2025}, we fine-tuned Llama3-8B for 10080 steps on the OBQA dataset using LoRA with the following hyperparameters: r= 8, alpha= 16, dropout: 0.1, targeting the query and value projections and the output head. It should be noted that for both implementations, we obtained results similar, though not identical, to those reported by the authors. Such variation can arise from stochastic training effects, environment-level nondeterminism common in LLM fine-tuning, and differences in compute backend and numerical precision (float32 on MPS versus the authors' bfloat16 CUDA setup). Exact replication of fine-tuning metrics was not the primary objective of this study, as our focus was on assessing metric differences arising from class cardinality rather than differences in model training performance. Table~\ref{tab:implementation_full_details} provides a full overview of the implementation details for both implementations. Regarding the ARC datasets, Implementation B keys ARC options by letter, so questions using the numeric 1-4 encoding are excluded (22 of 1,172 ARC-Challenge and 98 of 2,376 ARC-Easy test questions, giving AUPR base rates of 0.303 and 0.180); Implementation A maps positionally and retains them (0.299, 0.174). Each implementation re-scores its own fixed example set under both K values, so the reported $\Delta$ values are unaffected. An important difference between Implementation A and B is that we observed that padding strategy can affect logit extraction when using fixed positional indexing (e.g., logits[:, -1]). To account for this, Implementation B standardises on left-padding to ensure alignment between sequence end and prediction position.

\begin{table}[!ht]
\centering
\caption{Implementation and evaluation details for both Implementation A and Implementation B.}
\label{tab:implementation_full_details}
\small
\begin{tabularx}{\linewidth}{>{\raggedright\arraybackslash}p{0.28\linewidth} X}
\hline
\textbf{Category} & \textbf{Details} \\
\hline

Base models & \texttt{meta-llama/Meta-Llama-3-8B}, \texttt{Qwen/Qwen3.5-9B-Base} \\

\hline
Hardware and precision &
Mac Studio (Apple M3 Ultra, 512\,GB unified memory), MPS backend; model loaded and fine-tuned in float32 (float16 on MPS overflows the EDL loss terms and yields NaN gradients). Device and dtype are selected automatically: CUDA/bfloat16, MPS/float32, CPU/float32. \\

\hline
Average Time to Fine-Tune &
Llama3-8B: 1 hour 40 minutes; Qwen3.5-9B-Base: 3 hours. \\

\hline
Training dataset & \texttt{allenai/openbookqa}, train split \\

Validation dataset & \texttt{allenai/openbookqa}, validation split \\

Inference datasets & 
ID: \texttt{allenai/openbookqa}, test split. 
OOD: \texttt{allenai/ai2\_arc} ARC-Challenge test split; 
\texttt{allenai/ai2\_arc} ARC-Easy test split; 
\texttt{cais/mmlu} with \texttt{college\_mathematics}, \texttt{high\_school\_mathematics}, and \texttt{abstract\_algebra} test splits; 
\texttt{commonsense\_qa}, validation split. \\

\hline
Prompt template & 
\texttt{Return the label of the correct answer for the question below.}
\newline
\texttt{Question: \{question\}}
\newline
\texttt{Choices:}
\newline
\texttt{A) ...}
\newline
\texttt{B) ...}
\newline
\texttt{Answer:} \\

\hline
Answer token set & 
OBQA, ARC, and MMLU: \texttt{ A}, \texttt{ B}, \texttt{ C}, \texttt{ D}. 
CSQA: \texttt{ A}, \texttt{ B}, \texttt{ C}, \texttt{ D}, \texttt{ E}. \\

\hline
Target logits & 
IB-EDL: last-token logits at target IDs; tail logits, corresponding to the last $K$ positions, are passed through Softplus to produce per-class $\sigma$ for IB noise injection. 
Standard EDL: last-token logits at target IDs are passed through Softplus directly to produce evidence. \\

\hline
LoRA configuration & 
Target modules: \texttt{q\_proj}, \texttt{v\_proj}, \texttt{lm\_head}. 
Rank $r=8$; LoRA $\alpha=16$; dropout $=0.1$; bias = \texttt{lora\_only}; task type = \texttt{CAUSAL\_LM}. \\

\hline
Training steps & 
10080 steps \\

\hline
Batch size & 
Per-device train batch size: 4; per-device evaluation batch size: 4; gradient accumulation steps: 1. \\

\hline
Learning rate and optimisation & 
Learning rate: $5 \times 10^{-5}$; scheduler: cosine; warmup steps: 20; weight decay: 0.001; maximum gradient norm: 20.0. \\

\hline
Seed &
Implementation A: 42. Implementation B: 42. All random number generators (Python, NumPy, and PyTorch on CPU/CUDA/MPS) are seeded explicitly, and the \texttt{TrainingArguments} \texttt{seed} and \texttt{data\_seed} are set to 42; the same seed is applied at inference so the IB-EDL noise draws are deterministic. \\

\hline
Prediction checkpoint & 
Predictions are generated from the same final checkpoint for each method: \texttt{checkpoint-10080}. \\

\hline
Metric orientation & 
Accuracy: higher is better. NLL: lower is better. ECE: lower is better. ECE computed with 15 bins (following \cite{li_calibrating_2025}). \\

\hline
$K$ per dataset &
OBQA: $K=4$; ARC-Challenge: $K=4$; ARC-Easy: $K=4$; MMLU Math: $K=4$; CSQA: primary setting $K=5$, also evaluated as $K=4$. For CSQA, two output files were generated per method: five-class A--E and four-class A--D. In both codebases the evaluated class cardinality (\texttt{n\_labels}) is exposed as a configurable option, so each OOD dataset can be scored under both its correct $K$ and the mismatched $K$ used by Implementation A (e.g. ARC at $K=5$); for the four-class evaluations, examples whose correct answer lies outside the first $K$ options are filtered out. \\

\hline
\end{tabularx}
\end{table}

\subsection{Implementation A}
Implementation A was carried out using the code and implementation procedure provided by the authors of IB-EDL \cite{li_calibrating_2025}. With this configuration (that utilises the specifications provided above) we achieved similar results to the authors as seen in Table~\ref{tab:edl_metrics_impA}.

\begin{table}[!ht]
\centering
\caption{Fine-tuning Results for Implementation A: Llama3-8B fine-tuned using the authors of IB-EDL's provided code with IB-EDL and Standard EDL on the OBQA MCQA dataset.}
\label{tab:edl_metrics_impA}
\begin{tabular}{lccc}
\hline
\textbf{Method} & \textbf{Accuracy} & \textbf{ECE} & \textbf{NLL} \\
\hline
IB-EDL        & 0.882 & 0.031 & 0.453 \\
Standard EDL & 0.880 & 0.043 & 0.442 \\
\hline
\end{tabular}
\end{table}

It is important to note that IB-EDL can make use of an additional optional hyperparameter: $\varsigma$. When running the implementation for the experiments presented above, in accordance with the authors implementation for their main experiment, we used $\varsigma=0.0$. 

\subsection{Implementation B}
Implementation B utilised our own separate implementation of IB-EDL and standard EDL - both for the fine-tuning of Llama3-8B and the generation of predictions. As mentioned, we use the precise hyperparameters as used by IB-EDL and implementation specifics such as prompt formatting. With our implementation, we were similarly able to achieve close to the reported results for Llama3-8B fine-tuned using standard EDL and IB-EDL (Table~\ref{tab:edl_metrics_impB}).

\begin{table}[!ht]
\centering
\caption{Fine-tuning Results for Implementation B: Llama3-8B fine-tuned using our code with IB-EDL and Standard EDL on the OBQA MCQA dataset.}
\label{tab:edl_metrics_impB}
\begin{tabular}{lccc}
\hline
\textbf{Method} & \textbf{Accuracy} & \textbf{ECE} & \textbf{NLL} \\
\hline
IB-EDL        & 0.876 & 0.028 & 0.439 \\
Standard EDL & 0.886 & 0.041 & 0.425 \\
\hline
\end{tabular}
\end{table}

\section{Detailed OOD Detection Results}
\label{sec:detailed_results}
Tables~\ref{tab:ib_edl_num_classes_delta} - \ref{tab:own_std_edl_num_classes_delta} give the full MP, UM, and AUPR results (with per-row $\Delta$) underlying the UM-AUROC summary in Table~\ref{tab:um_summary}, for both implementations and both EDL variants \footnote{MMLU Math with $K=5$ is provided as further evidence of how mismatched cardinality can inflate AUROC and AUPR. The mismatched $K$ in prior work is constrained to the ARC datasets and CSQA.}.

\begin{table}[!ht]
\centering
\small
\setlength{\tabcolsep}{3pt}
\caption{Implementation A: IB-EDL OOD detection results for Llama3-8B fine-tuned on OBQA
($\texttt{sigma\_mult}=0.0$). $\Delta$ denotes the performance change
(correct/removed "E" minus mismatch/original). AUPR$_\text{base}$ is the ID-ratio baseline.}
\label{tab:ib_edl_num_classes_delta}
\begin{tabular}{llc cccc cccc}
\toprule
 &  &  & \multicolumn{4}{c}{MP} & \multicolumn{4}{c}{UM} \\
\cmidrule(lr){4-7}\cmidrule(lr){8-11}
Dataset & K & AUPR$_\text{base}$
& AUROC & $\Delta$ & AUPR & $\Delta$
& AUROC & $\Delta$ & AUPR & $\Delta$ \\
\midrule
MMLU Math & 5 (mismatch) & 0.516
& 0.961 & $\downarrow$ & 0.977 & $\downarrow$
& 0.979 & $\downarrow$ & 0.986 & $\downarrow$ \\
MMLU Math & 4 (correct) & 0.516
& 0.914 & -0.047 & 0.930 & -0.047
& 0.865 & -0.114 & 0.826 & -0.160 \\
\midrule
ARC-C & 5 (mismatch) & 0.299
& 0.937 & $\downarrow$ & 0.948 & $\downarrow$
& 0.967 & $\downarrow$ & 0.966 & $\downarrow$ \\
ARC-C & 4 (correct) & 0.299
& 0.674 & -0.263 & 0.492 & -0.456
& 0.670 & -0.297 & 0.471 & -0.495 \\
\midrule
ARC-E & 5 (mismatch) & 0.174
& 0.933 & $\downarrow$ & 0.933 & $\downarrow$
& 0.961 & $\downarrow$ & 0.943 & $\downarrow$ \\
ARC-E & 4 (correct) & 0.174
& 0.595 & -0.338 & 0.244 & -0.689
& 0.596 & -0.365 & 0.238 & -0.705 \\
\midrule
CSQA & 5 (as-is) & 0.291
& 0.936 & $\downarrow$ & 0.946 & $\downarrow$
& 0.941 & $\downarrow$ & 0.886 & $\downarrow$ \\
CSQA & 4 (removed "E") & 0.337
& 0.624 & -0.312 & 0.435 & -0.511
& 0.597 & -0.344 & 0.402 & -0.484 \\
\bottomrule
\end{tabular}
\end{table}

\begin{table}[!ht]
\centering
\small
\setlength{\tabcolsep}{3pt}
\caption{Implementation A: Standard EDL OOD detection results for Llama3-8B fine-tuned on OBQA.
$\Delta$ denotes the performance change (correct/removed "E" minus mismatch/original).
AUPR$_\text{base}$ is the ID-ratio baseline.}
\label{tab:standard_edl_num_classes_delta}
\begin{tabular}{llc cccc cccc}
\toprule
 &  &  & \multicolumn{4}{c}{MP} & \multicolumn{4}{c}{UM} \\
\cmidrule(lr){4-7}\cmidrule(lr){8-11}
Dataset & K & AUPR$_\text{base}$
& AUROC & $\Delta$ & AUPR & $\Delta$
& AUROC & $\Delta$ & AUPR & $\Delta$ \\
\midrule
MMLU Math & 5 (mismatch) & 0.516
& 0.950 & $\downarrow$ & 0.965 & $\downarrow$
& 0.951 & $\downarrow$ & 0.965 & $\downarrow$ \\
MMLU Math & 4 (correct) & 0.516
& 0.912 & -0.038 & 0.925 & -0.040
& 0.918 & -0.032 & 0.928 & -0.037 \\
\midrule
ARC-C & 5 (mismatch) & 0.299
& 0.879 & $\downarrow$ & 0.894 & $\downarrow$
& 0.866 & $\downarrow$ & 0.876 & $\downarrow$ \\
ARC-C & 4 (correct) & 0.299
& 0.603 & -0.277 & 0.378 & -0.516
& 0.601 & -0.265 & 0.373 & -0.503 \\
\midrule
ARC-E & 5 (mismatch) & 0.174
& 0.871 & $\downarrow$ & 0.867 & $\downarrow$
& 0.855 & $\downarrow$ & 0.840 & $\downarrow$ \\
ARC-E & 4 (correct) & 0.174
& 0.547 & -0.324 & 0.213 & -0.654
& 0.546 & -0.309 & 0.207 & -0.632 \\
\midrule
CSQA & 5 (as-is) & 0.291
& 0.888 & $\downarrow$ & 0.896 & $\downarrow$
& 0.873 & $\downarrow$ & 0.873 & $\downarrow$ \\
CSQA & 4 (removed "E") & 0.337
& 0.648 & -0.240 & 0.419 & -0.477
& 0.650 & -0.222 & 0.427 & -0.446 \\
\bottomrule
\end{tabular}
\end{table}

\begin{table}[!ht]
\centering
\small
\setlength{\tabcolsep}{2.5pt}
\caption{Implementation B: IB-EDL OOD detection results for Llama3-8B fine-tuned on OBQA.
$\Delta$ denotes the performance change (correct / removed "E" minus mismatch / as-is).
AUPR$_\text{base}$ is the ID-ratio baseline.}
\label{tab:own_ibedl_num_classes_delta}
\begin{tabular}{llc cccc cccc}
\toprule
 &  &  & \multicolumn{4}{c}{MP} & \multicolumn{4}{c}{UM} \\
\cmidrule(lr){4-7}\cmidrule(lr){8-11}
Dataset & K & AUPR$_\text{base}$
& AUROC & $\Delta$ & AUPR & $\Delta$
& AUROC & $\Delta$ & AUPR & $\Delta$ \\
\midrule
MMLU Math & 5 (mismatch) & 0.516
& 0.870 & - & 0.902 & -
& 0.939 & $\downarrow$ & 0.869 & $\downarrow$ \\
MMLU Math & 4 (correct) & 0.516
& 0.870 & 0.000 & 0.902 & 0.000
& 0.723 & -0.216 & 0.634 & -0.235 \\
\midrule
ARC-C & 5 (mismatch) & 0.303
& 0.629 & - & 0.429 & -
& 0.976 & $\downarrow$ & 0.954 & $\downarrow$ \\
ARC-C & 4 (correct) & 0.303
& 0.629 & 0.000 & 0.429 & 0.000
& 0.631 & -0.345 & 0.411 & -0.543 \\
\midrule
ARC-E & 5 (mismatch) & 0.180
& 0.519 & - & 0.191 & -
& 0.972 & $\downarrow$ & 0.952 & $\downarrow$ \\
ARC-E & 4 (correct) & 0.180
& 0.519 & 0.000 & 0.191 & 0.000
& 0.545 & -0.427 & 0.207 & -0.745 \\
\midrule
CSQA & 5 (as-is) & 0.291
& 0.894 & $\downarrow$ & 0.909 & $\downarrow$
& 0.962 & $\downarrow$ & 0.940 & $\downarrow$ \\
CSQA & 4 (removed "E") & 0.337
& 0.677 & -0.217 & 0.514 & -0.395
& 0.687 & -0.275 & 0.489 & -0.451 \\
\bottomrule
\end{tabular}
\end{table}

\begin{table}[!ht]
\centering
\small
\setlength{\tabcolsep}{2.5pt}
\caption{Implementation B: Standard EDL OOD detection results for Llama3-8B fine-tuned on OBQA.
$\Delta$ denotes the performance change
(correct / removed "E" minus mismatch / as-is).
AUPR$_\text{base}$ is the ID-ratio baseline.}
\label{tab:own_std_edl_num_classes_delta}
\begin{tabular}{llc cccc cccc}
\toprule
 &  &  & \multicolumn{4}{c}{MP} & \multicolumn{4}{c}{UM} \\
\cmidrule(lr){4-7}\cmidrule(lr){8-11}
Dataset & K & AUPR$_\text{base}$
& AUROC & $\Delta$ & AUPR & $\Delta$
& AUROC & $\Delta$ & AUPR & $\Delta$ \\
\midrule
MMLU Math & 5 (mismatch) & 0.516
& 0.891 & - & 0.885 & -
& 0.952 & $\downarrow$ & 0.965 & $\downarrow$ \\
MMLU Math & 4 (correct) & 0.516
& 0.891 & 0.000 & 0.885 & 0.000
& 0.898 & -0.054 & 0.890 & -0.075 \\
\midrule
ARC-C & 5 (mismatch) & 0.303
& 0.581 & - & 0.350 & -
& 0.859 & $\downarrow$ & 0.863 & $\downarrow$ \\
ARC-C & 4 (correct) & 0.303
& 0.581 & 0.000 & 0.350 & 0.000
& 0.581 & -0.278 & 0.346 & -0.517 \\
\midrule
ARC-E & 5 (mismatch) & 0.180
& 0.493 & - & 0.179 & -
& 0.840 & $\downarrow$ & 0.813 & $\downarrow$ \\
ARC-E & 4 (correct) & 0.180
& 0.493 & 0.000 & 0.179 & 0.000
& 0.494 & -0.346 & 0.179 & -0.634 \\
\midrule
CSQA & 5 (as-is) & 0.291
& 0.871 & $\downarrow$ & 0.881 & $\downarrow$
& 0.845 & $\downarrow$ & 0.810 & $\downarrow$ \\
CSQA & 4 (removed "E") & 0.337
& 0.602 & -0.269 & 0.402 & -0.479
& 0.593 & -0.252 & 0.385 & -0.425 \\
\bottomrule
\end{tabular}
\end{table}

\section{Entropy Analysis}
\label{sec:ent_analysis}
Tables~\ref{tab:ne_std_ibedl_delta_implA} and \ref{tab:ne_std_ibedl_imp_b} show how the AUROC and AUPR for the normalised Shannon entropy of the probabilities reflects the results seen as part of that main text - that, when class cardinality is matched appropriately (Implementation B), the entropy does not change. This provides supporting evidence of sensitivity to K-variance.

\begin{table}[ht]
\centering
\small
\setlength{\tabcolsep}{3pt}
\caption{Implementation A: OOD detection performance using normalised entropy ($H/\log_2 K$).
Results are shown for mismatch / as-is versus correct / removed "E" configurations.
$\Delta$ denotes the performance change (correct / removed "E" minus mismatch / as-is).}
\label{tab:ne_std_ibedl_delta_implA}
\begin{tabular}{llcccccccc}
\toprule
 &  & \multicolumn{4}{c}{Standard EDL} & \multicolumn{4}{c}{IB-EDL} \\
\cmidrule(lr){3-6}\cmidrule(lr){7-10}
Dataset & K
& AUROC & $\Delta$ & AUPR & $\Delta$
& AUROC & $\Delta$ & AUPR & $\Delta$ \\
\midrule
MMLU Math & 5 (mismatch)
& 0.941 & $\downarrow$ & 0.958 & $\downarrow$
& 0.955 & $\downarrow$ & 0.972 & $\downarrow$ \\
MMLU Math & 4 (correct)
& 0.913 & -0.028 & 0.925 & -0.033
& 0.916 & -0.039 & 0.931 & -0.041 \\
\midrule
ARC-C & 5 (mismatch)
& 0.841 & $\downarrow$ & 0.850 & $\downarrow$
& 0.914 & $\downarrow$ & 0.924 & $\downarrow$ \\
ARC-C & 4 (correct)
& 0.603 & -0.238 & 0.378 & -0.472
& 0.675 & -0.239 & 0.492 & -0.432 \\
\midrule
ARC-E & 5 (mismatch)
& 0.827 & $\downarrow$ & 0.808 & $\downarrow$
& 0.906 & $\downarrow$ & 0.894 & $\downarrow$ \\
ARC-E & 4 (correct)
& 0.547 & -0.280 & 0.213 & -0.595
& 0.595 & -0.311 & 0.244 & -0.650 \\
\midrule
CSQA & 5 (as-is)
& 0.855 & $\downarrow$ & 0.854 & $\downarrow$
& 0.911 & $\downarrow$ & 0.918 & $\downarrow$ \\
CSQA & 4 (removed "E")
& 0.648 & -0.207 & 0.419 & -0.435
& 0.624 & -0.287 & 0.435 & -0.483 \\
\bottomrule
\end{tabular}
\end{table}

\newpage
\begin{table}[!htbp]
\centering
\small
\setlength{\tabcolsep}{3pt}
\caption{Implementation B: OOD detection performance using normalised entropy ($H/\log_2 K$).
Results are shown for mismatch / as-is versus correct / removed "E" configurations.
$\Delta$ denotes the performance change (correct / removed "E" minus mismatch / as-is).}
\label{tab:ne_std_ibedl_imp_b}
\begin{tabular}{llcccccccc}
\toprule
 &  & \multicolumn{4}{c}{Standard EDL} & \multicolumn{4}{c}{IB-EDL} \\
\cmidrule(lr){3-6}\cmidrule(lr){7-10}
Dataset & K
& AUROC & $\Delta$ & AUPR & $\Delta$
& AUROC & $\Delta$ & AUPR & $\Delta$ \\
\midrule
MMLU Math & 5 (mismatch)
& 0.892 & - & 0.885 & -
& 0.874 & - & 0.903 & - \\
MMLU Math & 4 (correct)
& 0.892 & 0.000 & 0.885 & 0.000
& 0.874 & 0.000 & 0.903 & 0.000 \\
\midrule
ARC-C & 5 (mismatch)
& 0.581 & - & 0.350 & -
& 0.630 & - & 0.429 & - \\
ARC-C & 4 (correct)
& 0.581 & 0.000 & 0.350 & 0.000
& 0.630 & 0.000 & 0.429 & 0.000 \\
\midrule
ARC-E & 5 (mismatch)
& 0.493 & - & 0.179 & -
& 0.520 & - & 0.191 & - \\
ARC-E & 4 (correct)
& 0.493 & 0.000 & 0.179 & 0.000
& 0.520 & 0.000 & 0.191 & 0.000 \\
\midrule
CSQA & 5 (as-is)
& 0.830 & $\downarrow$ & 0.825 & $\downarrow$
& 0.863 & $\downarrow$ & 0.866 & $\downarrow$ \\
CSQA & 4 (removed "E")
& 0.602 & -0.228 & 0.402 & -0.423
& 0.678 & -0.185 & 0.514 & -0.352 \\
\bottomrule
\end{tabular}
\end{table}

\clearpage

\section{Additional Backbone: Qwen3.5-9B-Base (Implementation B)}
\label{sec:qwen_impB}
To test whether the class-cardinality effect is specific to Llama3-8B, we repeated the full Implementation B pipeline (identical hyperparameters, seed, prompt formatting, and left-padding; fine-tuned in float32) with the Qwen3.5-9B-Base backbone. Table~\ref{tab:edl_metrics_impB_qwen} reports the fine-tuning metrics, and Tables~\ref{tab:own_ibedl_num_classes_delta_qwen}--\ref{tab:ne_std_ibedl_imp_b_qwen} report the OOD detection results. The same pattern holds: vacuity-based UM (and, for CSQA, MP and normalised entropy) collapses toward chance once the evaluated class cardinality is matched, while MP and normalised entropy are invariant to the numerator $K$ for ARC and MMLU Math. This indicates that the effect is a property of the evaluation procedure as opposed to a property of a particular backbone.

\begin{table}[!ht]
\centering
\caption{Fine-tuning Results for Implementation B (Qwen3.5-9B-Base): fine-tuned using our code with IB-EDL and Standard EDL on the OBQA MCQA dataset.}
\label{tab:edl_metrics_impB_qwen}
\begin{tabular}{lccc}
\hline
\textbf{Method} & \textbf{Accuracy} & \textbf{ECE} & \textbf{NLL} \\
\hline
IB-EDL        & 0.938 & 0.027 & 0.268 \\
Standard EDL & 0.938 & 0.050 & 0.246 \\
\hline
\end{tabular}
\end{table}

\begin{table}[!ht]
\centering
\small
\setlength{\tabcolsep}{2.5pt}
\caption{Implementation B (Qwen3.5-9B-Base): IB-EDL OOD detection results.
$\Delta$ denotes the performance change (correct / removed "E" minus mismatch / as-is).
AUPR$_\text{base}$ is the ID-ratio baseline.}
\label{tab:own_ibedl_num_classes_delta_qwen}
\begin{tabular}{llc cccc cccc}
\toprule
 &  &  & \multicolumn{4}{c}{MP} & \multicolumn{4}{c}{UM} \\
\cmidrule(lr){4-7}\cmidrule(lr){8-11}
Dataset & K & AUPR$_\text{base}$
& AUROC & $\Delta$ & AUPR & $\Delta$
& AUROC & $\Delta$ & AUPR & $\Delta$ \\
\midrule
MMLU Math & 5 (mismatch) & 0.516
& 0.925 & - & 0.933 & -
& 0.961 & $\downarrow$ & 0.929 & $\downarrow$ \\
MMLU Math & 4 (correct) & 0.516
& 0.925 & 0.000 & 0.933 & 0.000
& 0.717 & -0.244 & 0.601 & -0.328 \\
\midrule
ARC-C & 5 (mismatch) & 0.303
& 0.489 & - & 0.292 & -
& 0.987 & $\downarrow$ & 0.977 & $\downarrow$ \\
ARC-C & 4 (correct) & 0.303
& 0.489 & 0.000 & 0.292 & 0.000
& 0.483 & -0.504 & 0.296 & -0.681 \\
\midrule
ARC-E & 5 (mismatch) & 0.180
& 0.394 & - & 0.142 & -
& 0.985 & $\downarrow$ & 0.939 & $\downarrow$ \\
ARC-E & 4 (correct) & 0.180
& 0.394 & 0.000 & 0.142 & 0.000
& 0.406 & -0.579 & 0.150 & -0.789 \\
\midrule
CSQA & 5 (as-is) & 0.291
& 0.960 & $\downarrow$ & 0.968 & $\downarrow$
& 0.939 & $\downarrow$ & 0.745 & $\downarrow$ \\
CSQA & 4 (removed "E") & 0.337
& 0.657 & -0.303 & 0.473 & -0.495
& 0.666 & -0.273 & 0.457 & -0.288 \\
\bottomrule
\end{tabular}
\end{table}

\begin{table}[!ht]
\centering
\small
\setlength{\tabcolsep}{2.5pt}
\caption{Implementation B (Qwen3.5-9B-Base): Standard EDL OOD detection results.
$\Delta$ denotes the performance change (correct / removed "E" minus mismatch / as-is).
AUPR$_\text{base}$ is the ID-ratio baseline.}
\label{tab:own_std_edl_num_classes_delta_qwen}
\begin{tabular}{llc cccc cccc}
\toprule
 &  &  & \multicolumn{4}{c}{MP} & \multicolumn{4}{c}{UM} \\
\cmidrule(lr){4-7}\cmidrule(lr){8-11}
Dataset & K & AUPR$_\text{base}$
& AUROC & $\Delta$ & AUPR & $\Delta$
& AUROC & $\Delta$ & AUPR & $\Delta$ \\
\midrule
MMLU Math & 5 (mismatch) & 0.516
& 0.882 & - & 0.884 & -
& 0.945 & $\downarrow$ & 0.964 & $\downarrow$ \\
MMLU Math & 4 (correct) & 0.516
& 0.882 & 0.000 & 0.884 & 0.000
& 0.881 & -0.064 & 0.882 & -0.082 \\
\midrule
ARC-C & 5 (mismatch) & 0.303
& 0.464 & - & 0.272 & -
& 0.872 & $\downarrow$ & 0.895 & $\downarrow$ \\
ARC-C & 4 (correct) & 0.303
& 0.464 & 0.000 & 0.272 & 0.000
& 0.464 & -0.408 & 0.272 & -0.623 \\
\midrule
ARC-E & 5 (mismatch) & 0.180
& 0.377 & - & 0.136 & -
& 0.863 & $\downarrow$ & 0.867 & $\downarrow$ \\
ARC-E & 4 (correct) & 0.180
& 0.377 & 0.000 & 0.136 & 0.000
& 0.376 & -0.487 & 0.136 & -0.731 \\
\midrule
CSQA & 5 (as-is) & 0.291
& 0.912 & $\downarrow$ & 0.926 & $\downarrow$
& 0.885 & $\downarrow$ & 0.898 & $\downarrow$ \\
CSQA & 4 (removed "E") & 0.337
& 0.656 & -0.256 & 0.504 & -0.422
& 0.657 & -0.228 & 0.508 & -0.390 \\
\bottomrule
\end{tabular}
\end{table}

\clearpage
\begin{table}[!ht]
\centering
\small
\setlength{\tabcolsep}{3pt}
\caption{Implementation B (Qwen3.5-9B-Base): OOD detection performance using normalised entropy ($H/\log_2 K$).
$\Delta$ denotes the performance change (correct / removed "E" minus mismatch / as-is).}
\label{tab:ne_std_ibedl_imp_b_qwen}
\begin{tabular}{llcccccccc}
\toprule
 &  & \multicolumn{4}{c}{Standard EDL} & \multicolumn{4}{c}{IB-EDL} \\
\cmidrule(lr){3-6}\cmidrule(lr){7-10}
Dataset & K
& AUROC & $\Delta$ & AUPR & $\Delta$
& AUROC & $\Delta$ & AUPR & $\Delta$ \\
\midrule
MMLU Math & 5 (mismatch)
& 0.882 & - & 0.884 & -
& 0.925 & - & 0.933 & - \\
MMLU Math & 4 (correct)
& 0.882 & 0.000 & 0.884 & 0.000
& 0.925 & 0.000 & 0.933 & 0.000 \\
\midrule
ARC-C & 5 (mismatch)
& 0.464 & - & 0.272 & -
& 0.489 & - & 0.292 & - \\
ARC-C & 4 (correct)
& 0.464 & 0.000 & 0.272 & 0.000
& 0.489 & 0.000 & 0.292 & 0.000 \\
\midrule
ARC-E & 5 (mismatch)
& 0.377 & - & 0.136 & -
& 0.394 & - & 0.142 & - \\
ARC-E & 4 (correct)
& 0.377 & 0.000 & 0.136 & 0.000
& 0.394 & 0.000 & 0.142 & 0.000 \\
\midrule
CSQA & 5 (as-is)
& 0.892 & $\downarrow$ & 0.905 & $\downarrow$
& 0.956 & $\downarrow$ & 0.962 & $\downarrow$ \\
CSQA & 4 (removed "E")
& 0.656 & -0.236 & 0.504 & -0.401
& 0.657 & -0.299 & 0.473 & -0.489 \\
\bottomrule
\end{tabular}
\end{table}

\end{document}